\documentclass[12pt]{article}

\usepackage[margin=1in]{geometry}

\usepackage{amsmath, amssymb, amsthm, amssymb }
\usepackage{amsfonts}
\usepackage{graphicx}
\usepackage[frame,ps,matrix,arrow,curve,rotate]{xy}
\usepackage{enumerate}
\usepackage{hyperref}
\usepackage{enumitem}
\usepackage{MnSymbol}
\usepackage{thm-restate}

\usepackage{todonotes}

\usepackage{bbm}

\newtheorem{theorem}{Theorem}[section]

\newtheorem{corollary}[theorem]{Corollary}
\newtheorem{proposition}[theorem]{Proposition}

\theoremstyle{definition}
\newtheorem{definition}[theorem]{Definition}

\newtheorem{question}[theorem]{Question}

\theoremstyle{remark}
\newtheorem{remark}[theorem]{Remark}

\newcommand{\mc}[1]{\mathcal{#1}}

\newcommand{\bfSigma}{\boldsymbol{\Sigma}}

\newcommand{\bfPi}{\mathbf{\Pi}}

\makeatletter
\newcommand\mathcircled[1]{%
	\mathpalette\@mathcircled{#1}%
}
\newcommand\@mathcircled[2]{%
	\tikz[baseline=(math.base)] \node[draw,circle,inner sep=1pt] (math) {$\m@th#1#2$};%
}
\makeatother

\setenumerate[1]{label={(\arabic*)}}

\begin{document}
	
	\title{Computable learning of natural hypothesis classes}
	\author{Syed Akbari\thanks{Syed Akbari was supported by an REU at the University of Michigan funded by the University of Michigan LSA Internship Scholarship.} \and Matthew Harrison-Trainor\thanks{The second author was partially supported by the National Science Foundation under Grant DMS-2153823.}}
	
	\maketitle

    \begin{abstract}
        This paper is about the recent notion of computably probably approximately correct learning, which lies between the statistical learning theory where there is no computational requirement on the learner and efficient PAC where the learner must be polynomially bounded. Examples have recently been given of hypothesis classes which are PAC learnable but not computably PAC learnable, but these hypothesis classes are unnatural or non-canonical in the sense that they depend on a numbering of proofs, formulas, or programs. We use the on-a-cone machinery from computability theory to prove that, under mild assumptions such as that the hypothesis class can be computably listable, any natural hypothesis class which is learnable must be computably learnable. Thus the counterexamples given previously are necessarily unnatural.
    \end{abstract}

    \section{Introduction}

    In the setting of binary classification learning we consider Probably Approximately Correct (PAC)-learning. The learner must, on most training sets (\textit{probably}), learn the binary classification with only small error (\textit{approximately}). There have generally been two approaches. The first, due to Valiant \cite{Valiant1984ATO}, is \textit{efficient} PAC learning where there must be a polynomial-time learning algorithm whose error is also polynomially bounded. The second is the statistical Vapnik-Chervonenkis \cite{VC} theory which gives a classification, in terms of VC-dimension, of the hypothesis classes which may be PAC-learned. In this VC-theory, the learners are arbitrary functions and there is no requirement that the learners be implementable by an algorithm, let alone an efficient one.

    Starting with Agarwal, Ananthakrishnan, Ben-David, Lechner, and Urner \cite{OLWCL}, a recent series of papers \cite{OCLWCL,FWOS} has investigated an intermediate notion where the learner is required to be computable but without placing any resource bounds on the learner. Part of the motivation for this line of inquiry is the recent discovery that learnability in general \cite{LCBU} and PAC learning in particular \cite{UoL} is sensitive to set-theoretic considerations such as the continuum hypothesis. To try to avoid this, Agarwal et al.\ introduced a new notion of \textit{computably probably approximately correct (CPAC)}  learning. Trivially CPAC learning implies PAC learning, but in \cite{OLWCL} it is shown that there are PAC-learnable hypothesis classes, i.e., with finite VC-dimension, which are not CPAC-learnable. Their counterexample is the hypothesis class
    \[ \mc{H} :=
        \{ h_{i,j} : \text{the $i$th proof is a proof of the $j$th formula of arithmetic}\}\]
    where
    \[ h_{i,j} (k) = \begin{cases}
        1 & k = 2i \\ 
        1 & k = 2j+1 \\
        0 & \text{otherwise}
    \end{cases}.\]
    This counterexample is a rather unnatural one, not looking like the hypothesis classes one might want to learn in practice. In particular, there is no canonical meaning to ``the $n$th formula of arithmetic'' and what functions are in the hypothesis class depend on arbitrary choices like how we number the formulas of arithmetic. There is no real answer to ``what is the 5th formula of arithmetic'' but which functions are in the hypothesis class $\mc{H}$ depends on the answer to such questions. (Commonly given examples of learnable problems, like upper half planes, are on the other hand much more canonical.)

    In this paper, we investigate whether this unnaturalness is a necessary feature of any counterexample. To make this formal, we appeal to a well-established idea from computability theory of properties holding ``on-a-cone''. The full definitions are somewhat technical, and we treat it in more detail in Section \ref{sec:on-a-cone}. Recall that an oracle is a ``black box'' provided to a computer which is capable of solving certain (often non-computable) problems in a single operation. Given a proof/statement about computers and programs, its relativization to an oracle $x$ is the replacement of every instance of a computer or program in that proof/statement by a computer or program with oracle $x$. One of the more well-known uses of oracles in computer science is the Baker-Gill-Solovay theorem, which says that there are oracles $x$ with $\mathsf{P}^x = \mathsf{NP}^x$ and also oracles $y$ with $\mathsf{P}^y \neq \mathsf{NP}^y$. Here, e.g., $\mathsf{P}^x$ is the class of decision problems solvable in polynomial time by a machine with oracle $x$. The Baker-Gill-Solovay theorem is a no-go theorem. It says that the $\mathsf{P}$ vs $\mathsf{NP}$ problem cannot be solved by proof techniques that relativize. This is because such a resolution, say a proof that $\mathsf{P} \neq \mathsf{NP}$, would relativize to the oracle $x$ from the  Baker-Gill-Solovay theorem to prove that $\mathsf{P}^x \neq \mathsf{NP}^x$, contradicting the fact that $\mathsf{P}^x = \mathsf{NP}^x$. However almost all standard proof techniques relativize, and so $\mathsf{P}$ vs $\mathsf{NP}$ cannot be resolved by these standard techniques.
    
    Around the same time, Martin posed his famous conjecture about the Turing degrees of natural problems. To capture what was meant by ``natural'', Martin used an approach often called ``on a cone'' which we will make use of in this paper. (See, e.g., \cite{Mon19} for a recent exposition of Martin's conjecture and the general ``on a cone'' approach. The approach has been used in a number of recent papers, e.g., \cite{}.) The general idea is as follows. Just like almost all proof techniques can be relativized, most natural statements can be relativized. That is, given a statement $\mathsf{P}$ which can be relativized to an oracle $x$ to give the statement $\mathsf{P}^x$, if $\mathsf{P}$ is natural, then $\mathsf{P}$ is true if and only if $\mathsf{P}^x$ is true. (There are natural propositions that do not easily relativize, such as Hilbert's Tenth problem; we do not consider these.) While natural propositions maintain their truth after relativization, unnatural propositions might change their truth value as we relativize to different oracles. However, the results of Martin \cite{MartinTuring,BorelDet} say that for propositions $\mathsf{P}$ which are not set-theoretically complicated, as we relativize to more and more powerful oracles, eventually the truth value of $\mathsf{P}$ stabilizes to a limit. This is the truth value of $\mathsf{P}$ \textit{on a cone}. The idea is that if $\mathsf{P}$ is natural, then the truth value of $\mathsf{P}$ relativizes, and $\mathsf{P}$ is true if and only if $\mathsf{P}$ is true on a cone. Thus: \textit{If we prove a theorem on a cone, then the theorem holds computably (with no oracle) when applied to natural objects.} 

    We can now restate the informal question ``if a \textit{natural} hypothesis class is PAC-learnable, is it CPAC-learnable?'' as the formal question ``if a hypothesis class is PAC-learnable, are its relativizations to an oracle CPAC-learnable on a cone?''

    Our main results are that this is usually true under various mild assumptions. We state below the informal statements of our theorems (using the word natural) and re-use the same numbering for the formal statements (using on-a-cone) which appear later in this paper. The theorems all have the same flavour, but use slightly different assumptions. There are also certain technical issues we have not yet addressed, such as whether or not the sample size function must be computable; these will be clarified in the formal statements.

    \begin{theorem}[Informal]
        If a natural hypothesis class is c.e.-represented and PAC-learnable, then it is (properly) CPAC-learnable.
    \end{theorem}

    \noindent We state this theorem first as it is the one that directly addresses the examples of PAC-learnable but not CPAC-learnable hypothesis classes from \cite{OLWCL}. By a c.e.-represented hypothesis class we mean that one can effectively list the functions in the hypothesis class.

    In proving the above result, we realised that there were other assumptions under which one can prove the same (or similar) conclusions. 

    \begin{theorem}[Informal]
        If a natural hypothesis class is topologically closed and PAC-learnable, then it is (properly) CPAC-learnable.
    \end{theorem}

    \noindent The hypothesis class being topologically closed means that the limit of any sequence of functions in the hypothesis class is itself in the hypothesis class. This is a natural condition which would hold of many natural geometric examples, like half-planes.

    \begin{theorem}[Informal, AD]
        If a natural hypothesis class is PAC-learnable, then it is improperly CPAC-learnable.
    \end{theorem}

    \noindent Note that in this case we only get improper learnability. This theorem is proved under the assumption of the Axiom of Determinacy, which is incompatible with the Axiom of Choice. So while one might not take the Axiom of Determinacy to be true, one should think of this theorem as saying that strong set-theoretic assumptions would be required to get a natural example of a PAC-learnable but not improperly CPAC-learnable hypothesis class.

    On the other hand, without these assumptions, one can show that there are natural hypothesis classes that are PAC-learnable but not CPAC-learnable.

    \begin{theorem}[Informal]
        There are natural hypothesis classes, consisting only of computable functions, which are $F_\sigma$ and PAC-learnable but not CPAC-learnable.
    \end{theorem}
    
    \noindent While natural from the computational standpoint expounded above, such examples are of course less natural from the learning theory standpoint as one is much less likely to want to learn an $F_\sigma$ hypothesis class.

    \section{Preliminaries and notation}

    \subsection{Classical statistical learning theory}

    We begin with a brief overview of the PAC-learning framework and fix our notation. See, e.g., \cite{UML} for a recent text covering this material.
    
    Let $\mc{X} = \mathbb{N}$ stand for the \textit{domain} or \textit{feature space}. This represents the set of all objects that we want to classify, or more precisely, the set of all combinations of possible object features that we want to label correctly. We fix $\mc{X} = \mathbb{N}$ here with $\mathbb{N}$ representing the G\"{o}del encoding of the feature space we actually care about. Let $\mc{Y} = \{0,1\}$ be the \textit{label space}.

    A \textit{hypothesis} is a function $h : \mc{X} \mapsto \mc{Y}$, and a \textit{hypothesis class} is a collection of hypotheses $\mc{H} \subseteq \mc{Y}^\mc{X}$. Given a PAC learning problem, the hypothesis class can be thought of as representing your complete knowledge of the problem or encoding your inductive priors, because it includes all the functions you think could represent the actual relationship between the features and labels and excludes all the others. Hypothesis classes can also thought of as all the functions representable by some machine learning algorithm, e.g., neural networks, linear regression, SVMs, etc.

    The training data is commonly assumed to be generated from some distribution $\mc{D}$ over $\mc{X} \times \mc{Y}$, and the performance of a hypothesis $h$ on this distribution is denoted by $L_{\mc{D}}(h) = \Pr_{(x,y) \sim D} [h(x) \neq y]$ and is called the \textit{true error}. The training data, usually called a \textit{sample}, $S$, is a finite sequence of feature-label pairs, so $S \in \mc{S} = \bigcup_{n \in \mathbb{N}}(\mc{X} \times \mc{Y})^n$. We measure the error of a hypothesis on the sample, also called the \emph{empirical risk}, with $ L_{\mc{S}}(h) = \frac{1}{m} \sum_{i=1}^{m} \mathbbm{1} [h(x_i) \neq y_i]$. The goal for a learner is to find a hypothesis minimizing the true error, but the learner only has access to $S$ and not $\mc{D}$, so it must work with the empirical risk instead as a proxy.
    
    \begin{definition}
        We say a hypothesis class $\mc{H}$ is \textit{PAC learnable} if there exists a learner function $A : \mc{S} \mapsto \mc{H}$ and a sample size function $m_{\mc{H}} : (0,1)^{2} \mapsto \mathbb{N}$ such that for every distribution $\mc{D}$ over $\mc{X} \times \mc{Y}$, for every error $\epsilon$, confidence $\delta$ $\in (0,1)^{2}$, and for every $m \geq m_{\mc{H}}(\epsilon, \delta)$, we have that $P_{S \sim \mc{D}^m}[L_{\mc{D}}(h') \leq \min_{h \in \mc{H}} L_{\mc{D}}(h) + \epsilon\ ] \geq 1 - \delta$.
    \end{definition}

    \noindent By default, the learner must output hypotheses inside the hypothesis class, and this is called \emph{proper} PAC learning, but we can relax this assumption to allow it to output hypotheses even outside of the hypothesis class, and we call this case \emph{improper} PAC learning.

    There are two very powerful characterization of PAC learnable classes given by the fundamental theorem of statistical learning that will be relevant for us: the first, based on the finiteness of a special combinatorial measure, called the $VC$ dimension, and the second, on the existence of a particular type of learner, called an empirical risk minimizer. Their definitions are as follow.

    \begin{definition}
        Given some collection of points $X = \{x_1, x_2, ... , x_k \} \subseteq \mc{X}$, we say $\mc{H}$ \textit{shatters} $X$ if $\forall y_1, y_2, ..., y_k \in \mc{Y} \ \exists h \in \mc{H}$ such that $\bigwedge_{i=1}^n h(x_i) = y_i$.        
    \end{definition}

    \begin{definition}
        The \textit{$VC$ dimension} of $\mc{H}$, $VC(\mc{H})$, is the maximum size of the set $X \subseteq \mc{X}$ it can shatter. If $\mc{H}$ can shatter arbitrarily large sets, we say its $VC$ dimension is infinite.
    \end{definition}

    Intuitively, we can say that a collection of domain points is shattered by a hypothesis class if every collection of labels for those points is expressible by the hypothesis class, so the $VC$ dimension can be thought of as a measure of expressiveness for hypothesis classes. The more expressive a hypothesis class, the more competing hypotheses there are to consider, and the more it is prone to overfitting.

    \begin{definition}
        A learner $A : \mc{S} \mapsto \mc{H}$ is called an \textit{empirical risk minimizer (ERM)} for $\mc{H}$ if given any sample $S \in \mc{S}$ the learner outputs a hypothesis that with the minimum empirical risk, i.e., $A(S) = \arg \min_{h \in \mc{H}} L_{S}(h)$.
    \end{definition}

    \begin{theorem}[Fundamental theorem of statistical learning]
            A hypothesis class $\mc{H}$ is PAC learnable if and only if the VC-dimension of $\mc{H}$ is finite. Moreover, if $\mc{H}$ is PAC learnable then any empirical risk minimizer for $\mc{H}$ witnesses this.
    \end{theorem}

    \subsection{Computable learning theory}

    We are now ready to give the basic definitions and theorems of computable learning theory from \cite{OLWCL}. All of these definitions and theorems relativize to any oracle, and given an oracle $x$ we use a superscript $^x$ for the relativization, e.g., CPAC$^x$ is the relativization of CPAC.
    
    \begin{definition}
        A hypothesis class $\mc{H}\subseteq \mc{Y}^\mc{X}$ is called \emph{computably PAC learnable (or CPAC learnable)} if it is PAC learnable and there is a computable learner $A : \mc{S} \mapsto \mc{H}$ as witness.
    \end{definition}
    
    \begin{definition}
        A hypothesis class $\mc{H}\subseteq \mc{Y}^\mc{X}$ is called \emph{strongly computably PAC learnable (or SCPAC learnable)} if it is PAC learnable and there is a computable learner $A : \mc{S} \mapsto \mc{H}$ and a computable sample size function witnessing this.
    \end{definition}

    We also extend the notions of \emph{proper} and \emph{improper} PAC learning to CPAC and SCPAC learning in the natural way. There are two more essential ideas we must define that characterize these various types of computable learnability before we can present the key theorems from the preceding papers that we make use of.

    \begin{definition}
        Let $\mc{H} \subseteq \mc{Y}^\mc{X}$ be a hypothesis class with $VC(\mc{H}) \leq d \in \mathbb{N}$. A \emph{$d$-witness} for $\mc{H}$ is a function $w : \mc{X}^{d+1} \mapsto \mc{Y}^{d+1}$ such that for each $\bar{u}$, $w(\bar{u}) = \bar{\ell}$ for some $\bar{\ell}$ such that there is no $h \in \mc{H}$ with $h(x_i) = \ell_i$ for all $i$.
    \end{definition}

    If $VC(\mc{H}) = d \in \mathbb{N}$, recall that means that no $d+1$ dimensional point can get shattered by $\mc{H}$, so for all $d+1$ dimensional points there exists some labeling not expressible by $\mc{H}$. A $d$-witness $w$ is just a function that takes every $d+1$ dimensional point to this label, hence witnessing that $VC(\mc{H}) \leq d$. If $VC(\mc{H})$ is infinite then there is no $d$-witness function for any $d \in \mathbb{N}$.

    \begin{definition}\label{defn:ERM}
        A learner $A : \mc{S} \mapsto \mc{H}$ is called an \emph{asymptotic empirical risk minimizer (asymptotic ERM)} for $\mc{H}$ if it outputs only hypotheses in $\mc{H}$, and if there is an infinite sequence $\{\epsilon_m \in [0,1]\}_{m=1}^{\mathbb{N}}$ converging to 0 as $m \mapsto \mathbb{N}$ such that for every sample $S \in \mc{S}$ we have that:
        \[L_S(A(S)) \leq \inf_{h \in \mc{H}}L_S(h) + \epsilon_{|S|}\]
    \end{definition}

    We can now present the effective versions of the fundamental theorem of statistical learning. The different notions of computable learning are no longer equivalent.

    \begin{theorem}[Effective fundamental theorem of statistical learning, \cite{OCLWCL,FWOS}]\label{keyPrevResultsThrm}
        Let $\mc{H} \subseteq \mc{Y}^\mc{X}$ be a hypothesis class of finite VC dimension. Then $\mc{H}$ is:
        \begin{enumerate}
            \item proper SCPAC learnable $\Longleftrightarrow$ there is a computable $ERM$ for $\mc{H}$
            \item proper CPAC learnable $\Longleftrightarrow$ there is a computable asymptotic $ERM$ for $\mc{H}$
            \item improper SCPAC learnable $\Longleftrightarrow$ improper CPAC learnable $\Longleftrightarrow$ there is a computable $d$-witness function for $\mc{H}$ (for some $d \in \mathbb{N})$
        \end{enumerate}
    \end{theorem}

    Part (1) is Proposition 1 of \cite{OCLWCL}, (2) is Proposition 2 of \cite{FWOS}, and  (3) is Theorem 3 of \cite{FWOS}. Since (3) also proves that improper SCPAC and improper CPAC are equivalent conditions, we will simply write improper S/CPAC when discussing these conditions. Also, from Theorem 3 of \cite{OCLWCL} and Theorem 4 of \cite{FWOS}, PAC $\not\Longrightarrow$ improper S/CPAC $\not\Longrightarrow$ proper CPAC $\not\Longrightarrow$ proper SCPAC (though clearly each of the reverse implications hold), so these are all distinct classes.

    In the effective setting, one might want to require that the hypothesis class be effectively represented in some form. One cannot possibly computably learn a hypothesis class if none of the hypotheses are computable. In \cite{OLWCL}, the following notions of an effective hypothesis class we used. For these definitions (and for the remainder of the paper), fix a computable listing $(\Phi_e)_{e \in \mathbb{N}}$ of the computable functions.

    \begin{definition}
        A hypothesis class $\mc{H} \subseteq \mc{Y}^\mc{X}$ is \emph{computably enumerably representable} (c.e.r) if there exists c.e.\ set of Turing machines such that the set of functions they compute is equal to $\mc{H}$. Similarly, it is \emph{computably representable} (c.r.) if there exists a computable set of Turing machines such that the set of functions they compute is equal to $\mc{H}$.
    \end{definition}
    
    These are the relevant notions for Theorem \ref{thrm:5} below. However there are other notions one could consider, such as $\mc{H}$ being given by an effectively closed set as represented by the infinite paths through a computable tree. For such a hypothesis class to be computably learnable, this tree would need at least some computable paths. This is the intended setting for Theorem \ref{thm:closed}.

    \begin{definition}
        A hypothesis class $\mc{H} \subseteq \mc{Y}^\mc{X}$ is \emph{effectively closed} if there is a computable tree $T \subseteq \{0,1\}^* = 2^{< \mathbb{N}}$ such that $\mc{H} = [T]$ is the set of paths through $T$.
    \end{definition}

    \noindent Recall that $\mc{X} = \mathbb{N}$ and $\mc{Y} = \{0,1\}$, so the paths through $T$ are hypotheses $\mc{X} \to \mc{Y}$.

    \section{The ``on-a-cone'' approach}\label{sec:on-a-cone}

    We begin by fixing some computability-theoretic notation. We use $2^{< \mathbb{N}}$ for the collection of finite binary strings, often denoted $\{0,1\}^*$, and $2^\mathbb{N}$ for Cantor space, the collection of infinite binary strings. We use lower case letters $b,x,y,z \in 2^{\mathbb{N}}$ for infinite binary strings. We identify an infinite binary string $x \in 2^{\mathbb{N}}$ with the corresponding subset $\{i \mid x(i) = 1\}$ of $\mathbb{N}$. Given $x,y \in 2^{\mathbb{N}}$ we write $x \leq_T y$ and say that $x$ is Turing-reducible to $y$ if we can compute $x$ given an oracle $y$. If $x \leq_T y$ and $y \leq_T x$ we say that $x$ and $y$ are Turing-equivalent and write $x \equiv_T y$. The Turing degrees are the equivalence classes of $2^{\mathbb{N}}$ modulo $\equiv_T$. Given $x,y \in 2^{\mathbb{N}}$ their join $x \oplus y$ is the set
    \[ x \oplus y = \{ (0,n) \; \mid \; n \in x\} \cup \{(1,n) \; \mid \; n \in y\}.\]
    This is a set of least Turing degree computing both $x$ and $y$. Given an infinite sequence of sets $(x_i)_{i \in \mathbb{N}}$, their join is
    \[ \bigoplus_{i \in \mathbb{N}} \{(i,n) \; \mid \; n \in x_i\}.\]

    We will also need some notions from descriptive set theory. First, let us recall the Borel sets. We work in Cantor space $2^\mathbb{N}$ with the standard topology. The basic clopen sets are the sets $[\sigma]$ of all extensions of a finite string $\sigma$, and the open sets are all unions of the basic clopen sets. The closed sets are their compliments, equivalently, a set is closed if and only if it is the set of infinite paths through a tree $T \subseteq 2^{< \mathbb{N}}$. The Borel sets are the smallest class of sets containing the open sets and closed under compliments and countable intersections and unions. One can think of the Borel sets are those with a reasonable definition.

    Given a set $A \subseteq 2^\mathbb{N}$, one can consider the two-player Gale-Stewart game $\mc{G}_A$ with $A$ as the payoff set. For this paper, one can take the precise definitions as a black box. This game is an infinitely long game of perfect information. Gale and Stewart showed that if $A$ is open or closed, then one of the two players has a winning strategy. However one can prove using the axiom of choice that there is a set $A$ such that neither player has a winning strategy in $\mc{G}_A$. If one of the two players has a winning strategy in $\mc{G}_A$, then we say that $A$ is \textit{determined}. Martin \cite{BorelDet} proved that all Borel sets are determined.

    \begin{theorem}[Borel determinacy, Martin \cite{BorelDet}]
        Every Borel set is determined.
    \end{theorem}

    Determinacy for more complicated sets is a further set-theoretic assumption. Analytic determinacy, for example, is equivalent to the existence of $0^\sharp$ \cite{MartinDet,Harrington}. The axiom of determinacy says that all sets are determined; it is incompatible with the axiom of choice. Intuitively however, one can think that the axiom of choice is required to construct an undetermined set.

    Given a set  $A \subseteq 2^\mathbb{N}$, we say that $A$ is \textit{degree-invariant} if whenever $x \in A$ and $y \equiv_T x$, $y \in A$. If $A$ is degree invariant, we can identify it with the corresponding set of Turing degrees $\{ \deg_T(x) \; \mid \; x \in A\}$.

\begin{definition}
    Given $x \subseteq \mathbb{N}$, the \textit{cone above $x$} is
    \[ C_x = \{ y \; \mid \; y \geq_T x \}.\]
\end{definition}

\begin{theorem}[Martin, \cite{MartinTuring}]
    If a degree-invariant set $A \subseteq 2^\mathbb{N}$ is determined, then it either contains a cone or is disjoint from a cone. In particular, every degree-invariant Borel subset of $2^\mathbb{N}$ either contains a cone or is disjoint from a cone.
\end{theorem}

Thinking of cones as large sets, one can define the $\{0,1\}$-valued Martin's measure on Borel degree-invariant sets by setting $\mu(A) = 1$ if $A$ contains a cone, and $\mu(A) = 0$ if $A$ is disjoint from a cone.

\begin{remark}\label{rem:ctble-additive}
    This measure is countably additive, so that if $\mu(\bigcup_i A_i) = 1$, and each $A_i$ is determined (e.g., Borel), then for some $i$, $\mu(A_i) = 1$. That is, if $\bigcup_i A_i$ contains a cone then, for some $i$, $A_i$ contains a cone. This will be an important property later in the paper.
\end{remark}

\noindent It is not hard to see why this remark is true. Suppose that for all $i$ $\mu(A_i) = 0$, i.e., that no $A_i$ contains a cone. For each $i$ there is a cone with base $b_i$ disjoint from $A_i$. Let $b = \bigoplus b_i$. Then the cone with base $b$ is disjoint from every $A_i$, and so disjoint from $\bigcup_i A_i$. Thus, being disjoint from a cone, $\mu(\bigcup_i A_i) = 0$.
    
    
\medskip
    
We can now consider learning on a cone. We begin with the example from \cite{OLWCL} of a PAC-learnable but not CPAC-learnable hypothesis class. Though that hypothesis class used provability in arithmetic, the example can be more simply formulated as follows. (This formulation was also used in \cite{OCLWCL} where Sterkenburg showed that it was not even improperly CPAC-learnable.) Let $K \subseteq \mathbb{N}$ be the halting problem $K = \{ e \; \mid \; \text{the $e$th program halts}\}$ (though any non-computable c.e.\ set could also be used). Then $K$ has a computable approximation $K_s$ where $K_s$ is the set of (indices for) programs which have halted after $s$ steps of computation. Let
    \[ \mc{H} :=
        \{ h_{s,e} : \text{$e \in K_s$}\}\]
    where
    \[ h_{s,e} (k) = \begin{cases}
        1 & k = 2e \\ 
        1 & k = 2s+1 \\
        0 & \text{otherwise}
    \end{cases}.\]
    Since $K_s$ is computable, we can computably list out $\mc{H}$, and $\mc{H}$ has VC dimension at most $2$. One can argue that $\mc{H}$ is not CPAC learnable. Towards a contradiction, assume that there is a proper CPAC learner for $\mc{H}$ in the agnostic setting. Given $e$, let $\mc{D}_i$ be a distribution with all weight on $(2e,1)$. Then almost surely any training sample contains is a sequence of $(2e,1)$ with label $1$. If $e \in K$, then for sufficiently long samples (of a fixed length known ahead of time) the learner must output a function $h_{s,e}$ where $e \in K_s$. Otherwise, it may output any function from $\mc{H}$. So we can use the learner to compute $K$: for each $e$, feed in a sufficiently long sample sequence consisting of $(2e,1)$ with label 1 to the learner, for which it outputs a function $h = h_{s,e'}$. If $h(2e) = 1$ (so that $e' = e$) then $e \in K$. Otherwise, $e \notin K$.

    One way of relativizing $\mc{H}$ is to consider the relativization of $K$ to an oracle $x$, $K^x = \{ e \; \mid \; \text{the $e$th program with oracle $x$ halts} \}$, and then define
    \[ \mc{H}^x = \{ h_{s,e} \; \mid \; e \in K_s \} .\] 
    The argument above relativizes to show that $\mc{H}^x$ is PAC-learnable, but not properly CPAC learnable relative to $x$ in the agnostic setting. If $x \equiv_T y$ are Turing-equivalent, then everything computable with oracle $x$ is computable with oracle $y$ and vice versa. However, if $x \equiv_T y$ are Turing-equivalent oracles, then even though $K^x \equiv_T K^y$ are Turing-equivalent, we might have $K^x \neq K^y$ and so $\mc{H}^x \neq \mc{H}^y$. This is because the $i$th program with oracle $x$ might be completely different from the $i$th program with oracle $y$, and so one might halt while the other does not. This is essentially capturing the fact that there is no canonical choice for listing out all programs; and the set $K$, and thus what functions $h_{s,i}$ are in $\mc{H}$, depends on this listing. If we choose a different listing, then the hypothesis class $\mc{H}$ obtained will be different. We will say that $x \mapsto \mc{H}^x$ is not degree invariant because $x \equiv_T y \not\Longrightarrow \mc{H}^x = \mc{H}^y$.

    Another way to relativize $\mc{H}$ is to ignore the oracle $x$, letting
    \[ \mc{H}^x = \{h_{s,e} \; \mid \; e \in K_s\}. \]
    This is degree-invariant, $x \equiv_T y \Longrightarrow \mc{H}^x = \mc{H}^y$, because $\mc{H}^x = \mc{H}^y$ for any $x$ and $y$. On the other hand, for any $x \geq_T K$, $\mc{H}^x$ is properly CPAC learnable in the agnostic setting relative to $x$. This is because the learner can use the oracle $x \geq_T K$ and so the learner knows, given $e$, whether or not $e \in K$. So to relativize the initial example $\mc{H}$, it seems that we must either make $\mc{H}^x$ non-degree-invariant, or make $\mc{H}^x$ properly CPAC learnable relative to $x$ for all oracles on the cone $C_K = \{ x \; \mid \; x \geq_T K\}$.

    On the other hand, consider a natural hypothesis class such as the set of computable upper half planes. This class relatives to, given an oracle $x$, the hypothesis class $\mc{H}^x$ of $x$-computable upper half planes. This is degree-invariant, because the $x$-computable upper half planes are the same as the $y$-computable upper half planes. For all $x$, $\mc{H}^x$ is PAC-learnable and CPAC-learnable relative to $x$.

    A general heuristic is that natural classes are given by properties of functions while unnatural classes are given as properties of programs or indices. Whether a function givens an upper half-plane or not does not depend on the program computing the function, but solely on the function itself. On the other hand, whether a hypothesis $h_{s,e}$ is in $\mc{H} = \{h_{s,e} \; \mid \; e \in K_s\} $ depends on the program $e$, and how long it takes to run.

Keeping these motivating examples in mind, we are now ready for the formal definitions.

\begin{definition}
        Let $x \mapsto \mc{H}_x$ be a function $2^\mathbb{N} \to \mc{P}(2^X)$ which associates to each oracle $x \subseteq \mathbb{N}$ a hypothesis class. We denote this function by $\mc{H}$ and say that $\mc{H}$ is a \textit{family of hypothesis classes}. We say that $\mc{H}$ is:
        \begin{enumerate}
            \item \textit{degree-invariant} if whenever $x \equiv_T y$, $\mc{H}_x = \mc{H}_y$.
            \item \textit{Borel} if its graph $\Gamma(\mc{H}) = \{ (x,f) : x \in \mc{H}_x \}$ is Borel.
            \item \textit{closed} if each $\mc{H}_x$ is closed (and similarly for open, $\bfPi^0_2$, and so on).
        \end{enumerate}
    \end{definition}

    \begin{definition}
        Given $\mc{H}_x$ a degree-invariant family of hypothesis classes, and $P$ a property of hypothesis class, we say that $\mc{H}_x$ is $\mathsf{P}$ on a cone if there is $y$ such that for all $z \geq_T y$, $\mc{H}_y$ has property $\mathsf{P}$. For example:
        \begin{enumerate}
            \item $\mc{H}_x$ is \textit{PAC-learnable on a cone} if there is a cone of $x$'s on which $\mc{H}_x$ is PAC-learnable, that is, there is a cone with base $b$ such that for all $x \geq_T b$, $\mc{H}_x$ is PAC-learnable.
            \item $\mc{H}_x$ is \textit{not-PAC-learnable on a cone} if there is a cone of $x$'s on which $\mc{H}_x$ is not PAC-learnable, that is, there is a cone with base $b$ such that for all $x \geq_T b$, $\mc{H}_x$ is not PAC-learnable.
            \item $\mc{H}_x$ is \textit{CPAC-learnable on a cone} if there is a cone of $x$'s on which $\mc{H}_x$ is CPAC$^x$-learnable, i.e., there is a cone with base $b$ such that for all $x \geq_T b$, there is an $x$-computable PAC-learner for $\mc{H}_x$.
            \item $\mc{H}_x$ is \textit{c.e.-represented on a cone} if there is a cone with base $b$ such that for all $x \geq_T b$, $\mc{H}_x$ is c.e.-represented relative to $x$, that is, represented by an $x$-c.e.\ set of $x$-computable functions.
        \end{enumerate}
    \end{definition}

    Note that $\mc{H}_x$ being not-PAC-learnable on a cone is not the same as it not being the case that $\mc{H}_x$ is PAC-learnable on a cone. This is because if $\mc{H}_x$ is not Borel, then
    \[ \{ x \; \mid \; \text{$\mc{H}_x$ is PAC-learnable}\}\]
    might not be Borel, and hence might both not contain a cone and not be disjoint from a cone. 

    In general, if $\mc{H}_x$ is a degree-invariant family of Borel hypothesis classes, and $P$ is a Borel property, then $ \{x \; \mid \; \text{$\mc{H}_x$ is $\mathsf{P}$}\}$
    is a Borel set and thus either contains a cone or is disjoint from a cone. Thus either $\mc{H}_x$ is $\mathsf{P}$ on a cone or $\mc{H}_x$ is not-$\mathsf{P}$ on a cone. On the other hand, if $\mc{H}_x$ is not Borel, or $\mathsf{P}$ is not Borel, then $ \{x \; \mid \; \text{$\mc{H}_x$ is $\mathsf{P}$}\}$ might not be a Borel set. In this case it is possible that it is not true that $\mc{H}_x$ is $\mathsf{P}$ on a cone and it is also not true that $\mc{H}_x$ is not-$\mathsf{P}$ on a cone. (It is never the case that $\mc{H}_x$ is $\mathsf{P}$ on a cone and that $\mc{H}_x$ is also $\neg \mathsf{P}$ on a cone.)
    
    If $\mc{H}_x$ is Borel, then because being PAC learnable is a Borel property, either $\mc{H}_x$ is PAC-learnable on a cone or $\mc{H}_x$ is not-PAC-learnable on a cone, and similarly for CPAC-learnability. One of the main technical issues in the results in this paper is whether the sets involved are determined and hence whether we can conclude that they contain or are disjoint from a cone.

    \section{Positive results}

    \subsection{C.e.-represented hypothesis classes}

    We begin with our theorem most directly addressing the results of \cite{OLWCL}, namely, the c.e.-represented case. We show that if a natural hypothesis class is c.e.-represented and PAC-learnable, then it is SCPAC-learnable.

    \begin{theorem}\label{thrm:5}
        Let $\mc{H}_x$ be a degree-invariant family of hypothesis classes that is c.e.-represented on a cone and PAC-learnable on a cone. Then for all $x$ on a cone, $\mc{H}_x$ is properly SCPAC.
    \end{theorem}
    \begin{proof}
        Since $\mc{H}_x$ is c.e.-represented on a cone (say the cone above $y_1$) and PAC-learnable on a cone (say the cone above $y_2$), then on the cone above $y_1 \oplus y_2$ $\mc{H}_x$ is both c.e.-represented and PAC-learnable. In the rest of the proof, for simplicity of notation, we will omit the oracle $y_1 \oplus y_2$ and assume that $\mc{H}_x$ is c.e.-represented and PAC-learnable for all $x$, and prove that for all $x$ the hypothesis class $\mc{H}_x$ is properly SCPAC learnable relative to $x$. Moreover, because Martin's measure is countably additive, we may assume that there is a single $e$ such that for all $x$ the class $\mc{H}_x$ is c.e.-represented by $W_e^x$ (see Remark \ref{rem:ctble-additive}). Indeed,
        \[ 2^\mathbb{N} = \bigcup_e \{ x \; \mid \; \text{$\mc{H}_x$ is represented by the c.e.\ set $W^x_e$}\}\]
        and each set in the union is Borel, and so one of these sets (say the one corresponding to $e$) must contain a cone. For all $x$ on this cone, $\mc{H}_x = \{ \Phi_n^x \; \mid \; n \in W_e^x\}$.
        
        Let $S \in \mc{S}$ be a sample of size $m$. We can partition the Turing degrees as follows:
        \[ 2^\mathbb{N}  = \bigcup_{i=1}^m \left\{x : \min_{h \in \mc{H}_x} L_{S}(h) = \frac{i}{m} \right\}.\]
        These sets are degree-invariant (because $\mc{H}_x$ is) and Borel because $\mc{H}_x$ is c.e.-represented by $W_e^x$. Then one of them contains a cone $C_S$ with base $b_S$, say the set corresponding to $i_S$.
        
        Define the function $f:\mc{S} \mapsto \mathbb{Q}$ such that $f(S) = \frac{i_{S}}{m}$. Then $f(S) = \min_{h \in \mc{H}_x} L_{S}(h)$ for all $x \geq_{T} b_{S}$. So, given any sample $S \in \mc{S}$, $f$ takes $S$ to the minimal empirical risk value that is reached on $S$ on the cone above $b_{S}$.

        Let $b = (\bigoplus_{S \in \mc{S}} b_{S}) \bigoplus f$ and let $x \geq_{T} b$. By the relativized version of Theorem \ref{keyPrevResultsThrm} (1), $\mc{H}_x$ is properly SCPAC$^x$ if and only if $\mc{H}_x$ is PAC learnable (which we know that is is) and there is an $x-$computable ERM for $\mc{H}_x$. Such an $x$-computable ERM exists: Given a sample $S \in \mc{S}$, compute $f(S)$, and look for an index $n \in W_e^x$ such that $L_S(\Phi_n^x) = f(S)$. This $\Phi_n^x \in \mc{H}_x$ is the hypothesis with the minimum empirical risk. Output $\Phi_n^x$.
    \end{proof}

    \subsection{Closed hypothesis classes}

    We now assuming that the hypothesis class is topologically closed, i.e., the limit of any sequence of hypotheses is again a hypothesis. We prove that if a natural hypothesis class is topologically closed and PAC-learnable, then it is SCPAC-learnable.

    \begin{theorem}\label{thm:closed}
        Let $\mc{H}_x$ be a degree-invariant Borel family of hypothesis classes such that for all $x$, $\mc{H}_x$ is closed and PAC-learnable. Then for all $x$ on a cone, $\mc{H}_x$ is properly SCPAC.
    \end{theorem}

    This follows from the following two propositions.

    \begin{proposition}
        Let $\mc{H}_x$ be a degree-invariant Borel family of hypothesis classes such that for all $x$, $\mc{H}_x$ is closed. Then there is a tree $T$ with no dead ends such that $\mc{H}_x = [T]$ for all $x$ on a cone.
    \end{proposition}
    \begin{proof}
        When we say that $\mc{H}_x$ is a Borel family, recall that we mean that the set $\mc{H} = \{ (x,h) : h \in \mc{H}_x \}$ is Borel. The sets $\mc{H}_x$ are the sections of this set, and are all closed.
        
        We use a theorem of Kunugi and Novikov from descriptive set theory (see Theorem 28.7 of \cite{Kechris}, particularly in the form of Exercise 28.9) which says that given a Borel set $\mc{H} \subseteq 2^\mathbb{N} \times 2^\mathbb{N}$ such that every section $\mc{H}_x$ is closed, there is a Borel map $x \mapsto T_x \subseteq 2^{< \mathbb{N}}$ where $T_x$ is a tree with $[T_x] = \mc{H}_x$. In the general form of this theorem, the sections $\mc{H}_x$ are subsets of Baire space $\mathbb{N}^\mathbb{N}$ and the trees $T_x$ might not be pruned. In our case, $\mc{H}_x$ is a subset of Cantor space $2^\mathbb{N}$ and so the trees $T_x$ are finitely branching and we may assume that they are pruned. This corresponds to replacing $T_x$ by 
        \[ T_x^* = \{ \sigma \in T \;\mid\; \text{for all $n \geq |\sigma|$, there is $\tau \in T_x$ with $\tau \succ \sigma$ of length $n$ }\}.\]
        Let
        \[ B_\sigma = \{ x \in 2^\mathbb{N} \; \mid \; \sigma \in T_x\}.\]
        Both $B_\sigma$ and its compliment are Borel and so one of the two contains a cone. Let $b_\sigma$ be the base of this cone.

        Let
        \[ T = \{ \sigma : \text{$B_\sigma$ contains a cone}\}.\]
        Let $b = \bigoplus b_\sigma$. Then, for all $x \geq_T b$,
        \[ \{ \sigma : \text{there is an $f \in \mc{H}_x$ with $f \succ \sigma$} \} = T.\]
        Since $\mc{H}_x$ is closed, $\mc{H}_x = [T]$.
    \end{proof}

    \begin{proposition}
        Let $\mc{H}$ be a closed hypothesis class such that
        \[ T = \{ \sigma \in 2^{< \mathbb{N}} : \exists h \in \mc{H} \; h \succ \sigma\}\]
        is computable. Then $\mc{H}$ is properly SCPAC.
    \end{proposition}
    \begin{proof}
        By Theorem \ref{keyPrevResultsThrm} (1), $\mc{H}$ is properly SCPAC learnable if and only if it is PAC learnable and there is a computable $ERM$ for $\mc{H}$. It is easy to compute an ERM for $\mc{H}$ using $T$ as follows. Given a sample $S$, the tree $T$ is enough to tell us all of the possible labellings of $S$ by functions in $\mc{H}$. We can pick one with the least error, and since $T$ has no dead ends, we can compute a path in $T$ agreeing with this labeling (say, the leftmost such path in $T$).
    \end{proof}

    \begin{proof}[Proof of Theorem \ref{thm:closed}]
        By the first proposition, there is a tree $T$ with no dead ends such that $\mc{H}_x = [T]$ for all $x$ on a cone, say with base $b$. Increasing the cone from $b$ to $b^* = b \oplus T$, for all $x$ on the cone above $b^*$, by Proposition 4.5 we have that $\mc{H}_x$ is properly SCPAC$^x$.
    \end{proof}

    Though the theorem of Kunugi and Novikov also holds for open sets, we do not get any useful theorem out of this as if a hypothesis class $\mc{H}$ is open and non-empty then $\mc{H}$ is not PAC. This is because there is some finite partial hypothesis $h_0$ such that every hypothesis extending $h_0$ is in $\mc{H}$. Thus $\mc{H}$ shatters arbitrarily large sets.
    
    \subsection{Arbitrary hypothesis classes under the axiom of determinacy}

    In this section we relax our setting to that of improper learning, i.e., the learner does not need to output a hypothesis from the hypothesis class. From a computational standpoint, this is much easier. We prove the following theorem under the axiom of determinacy; following it, we discuss what happens if the the family of hypothesis classes $\mc{H}_x$ is Borel.

    \begin{theorem}[ZF + Axiom of Determinacy]\label{thm:det}
        Let $\mc{H}_x$ be a degree-invariant family of hypothesis classes that is PAC-learnable on a cone. Then for all $x$ on a cone, $\mc{H}_x$ is improperly SCPAC.
    \end{theorem}
    \begin{proof}
        As we did previously, we may assume for simplicity that $\mc{H}_x$ is PAC-learnable for all $x$.

        Each $\mc{H}_x$ is PAC-learnable and hence has finite VC-dimension; moreover, by invariance, if $x \equiv_T y$, then $\mc{H}_x = \mc{H}_y$ have the same VC-dimension. Thus, for each $d$, the set $VC_{d}=\{x: VC(\mc{H}_x) = d\}$ is degree-invariant and Borel. It follows that there is a number $d$ and a cone on which all $\mc{H}_x$ have VC-dimension $d$. So as a simplification we may assume that all $\mc{H}_x$ have VC-dimension $d$.

        Given $x$, because $VC(\mc{H}_x)=d$, for each $\bar{u} \in \mathbb{N}^{d+1}$ there is a labeling $\bar{\ell} \in 2^{d+1}$ witnessing that $\bar{u}$ cannot be shattered by $\mc{H}_x$, i.e., the functions in $\mc{H}_x$ cannot give the labeling $i$ to $\bar{u}$. Moreover, there are finitely many such labels, so there is a lexicographically least such label. Call this labeling $w_x(\bar{u})$.

        Fix $\bar{u} \in \mathbb{N}^{d+1}$ and observe that we can partition the Turing degrees as follows:
        \[2 ^{\mathbb{N}}  = \bigcup_{\bar{\ell} \in 2^{d+1}} \{x : w_{x}(\bar{u})=\bar{\ell} \} \]
        These sets are degree-invariant and so (using the axiom of determinacy) there is a cone of $x$'s on which $w_x(\bar{u}) = \bar{\ell}$. Doing this for each $\bar{u}$, and taking the intersection of all of these cones, we get a function $w_{x}$ and a cone such that for all $x$ on that cone, $w_x = w$. Increasing the cone, we may also assume that $w$ is computable.

        By the relativized version of Theorem \ref{keyPrevResultsThrm} (3) $\mc{H}_x$ is improper SCPAC$^x$ learnable if and only if there is an $x$-computable witness function for $\mc{H}_x$. We have just shown that, for all $x$ on a cone, $w$ is an $x$-computable witness function and thus $\mc{H}_x$ is improper SCPAC$^x$ learnable. 
    \end{proof}

    If the hypothesis classes are Borel, then we require only analytic determinacy. The analytic sets are the continuous images of Borel sets, and analytic determinacy is not provable in ZFC (but has weak large cardinal strength, being equivalent to the existence of $x^\sharp$ for every $x$). This is widely believed to be consistent with ZFC. $V=L$ is an example of an extension of ZFC inconsistent with analytic determinacy.

    \begin{corollary}[ZF + Analytic Determinacy]\label{cor:det}
        Let $\mc{H}_x$ be a degree-invariant Borel family of hypothesis classes such that for all $x$ on a cone, $\mc{H}_x$ is PAC-learnable. Then for all $x$ on a cone, $\mc{H}_x$ is improperly SCPAC.
    \end{corollary}
    \begin{proof}
        Since each $\mc{H}_x$ is Borel, the sets
        \[\bigcup_{i \in 2^{d+1}} \{x : x\geq_{T} y \land w_{x}(\bar{u})=i \}\]
        are $\bfPi^1_1$ and hence determined by Analytic Determinacy. This suffices to complete the proof of the previous theorem without assuming any more determinacy.
    \end{proof}


    \section{Negative Results}

    In this section we give some negative results saying that there are examples of natural PAC learnable but not CPAC learnable classes not falling under the previous theorems. The reader should keep in mind that while these are in some sense natural from a computational standpoint, they are still, e.g., not c.e.-represented and are not exactly natural from the standpoint of learning theory.

    \begin{theorem}
        There is a degree-invariant Borel family of hypothesis classes $\mc{H}_x$ such that for all $x$, $\mc{H}_x$ is topologically $F_\sigma$ ($\bfSigma_2$), PAC-learnable, and consists only of $x$-computable functions, but $\mc{H}_x$ is not CPAC-learnable on any cone.
    \end{theorem}
    \begin{proof}
        Let $\mc{F}$ be the set of all functions $f \colon \mathbb{Q} \to \{0,1\}$ such that:
        \begin{enumerate}
            \item there are $p,q$ such that $f(p) = 0$ and $f(q) = 1$;
            \item if $p < q$, then if $f(p) = 1$ then $f(q) = 0$;
            \item there is no maximal element $p$ with $f(p) = 0$ and no minimal element $q$ with $f(q) = 1$.
        \end{enumerate}
        It is easy to see that $\mc{F}$ has VC dimension 1, so any subset of $\mc{F}$ will also have VC dimension 1. $\mc{F}$ is also Borel.
        
        To each $f \in \mc{F}$ we can associate a unique irrational $r$ such that $f(p) = 0$ for $p < r$ and $f(p) = 1$ for $p > r$. It is easy to see that $\mc{F}$ has VC dimension 1. Define, for each $x$, $\mc{H}_x = \{ f \in \mc{F} : f \equiv_T x\}$. Note that this is degree-invariant and that each $\mc{H}_x$ has VC dimension 1 (hence is PAC). Moreover, each $\mc{H}_x$ is countable and hence $F_\sigma$. 
        

        Now we argue that $\mc{H}_x$ is not CPAC on a cone. Suppose towards a contradiction that $\mc{H}_x$ was CPAC on a cone. Then, for each $x$ on a cone, there would be an $x$-computable asymptotic ERM for $\mc{H}_x$. Given an $x$-computable function which is a purported asymptotic ERM for $\mc{H}_x$, we can check in a Borel way whether it is in fact an asymptotic ERM (i.e., whether the computable function satisfies Definition \ref{defn:ERM}). Given $x$, choose one of these asymptotic ERMs for $\mc{H}_x$ as follows. List the samples $S \in \mc{S}$ as $S_1,S_2,\ldots$, and list the rationals as $q_1,q_2,q_3,\ldots$. Then we can think of an asymptotic ERM as a binary sequence $\mathbb{N} \times \mathbb{N} \to \{0,1\}$ via $(m,n) \mapsto (S_m,q_n) \mapsto A(S_m)(q_n)$. For each $x$, choose the lexicographically least asymptotic ERM $A_x$ for $\mc{H}_x$. If $x \equiv_T y$, then $A_x = A_y$, and because we chose $A_x$ and $A_y$ using only their properties as functions (and not the programs that compute them) we will have chosen the same asymptotic ERM $A_x = A_y$.

        Now for each sample $S$ and $q \in \mathbb{Q}$, consider
        \[ \{ x \; \mid \; A_x(S)(q) = 0\} \text{ and } \{ x \; \mid \; A_x(S)(q) = 1\}. \]
        This is degree-invariant (since the choice of $A_x$ is degree-invariant) One of these two must contain a cone, say with base $b_{S,q}$ and let $g_S(q)$ record the corresponding value $0$ or $1$. Let $b = \bigoplus b_{S,q}$. For all $x \geq_T b$, and all $S \in \mc{S}$, $A_x(S) = g_S$ is the same function. But for $x \nequiv_T y$, $\mc{H}_x$ and $\mc{H}_y$ are disjoint, and so for most $x \geq_T b$, $A_x(S) \notin \mc{H}_x$. Thus, for those $x$, $A_x$ cannot be an asymptotic ERM for $\mc{H}_x$. But of course $A_x$ was chosen to be an asymptotic ERM for all $x$, yielding a contradiction. So it is not true that $\mc{H}_x$ is CPAC on a cone.
    \end{proof}

    The hypothesis classes $\mc{H}_x$ are $\bfSigma^0_2$/$F_\sigma$ because they are countable, but this is making strong use of the fact that this is boldface $\bfSigma^0_2$ rather than lightface $\Sigma^0_2$. (The family $\mc{F}$ is not $\bfSigma^0_2$.) We already know what happens if the hypothesis class is countably and computably listable, but we can ask about conditions in between, e.g., $\Sigma^0_2$-represented hypothesis classes as we will now define.
    
    Recall that the $\Sigma^0_2$ subsets of $\mathbb{N}$ are the projections onto the first coordinate of a co-c.e.\ set $W \subseteq \mathbb{N} \times \mathbb{N}$. We can form a listing of the $\Sigma^0_2$ sets
     \[ U_e = \{ m \in \mathbb{N} \;\mid\; \text{$\exists n$ such that $\Phi_e(m,n)$ does not halt}\}.\]
    This also relativizes to give a listing $U_e^x$ of the $\Sigma^0_2(x)$ sets.

     \begin{definition}
         A hypothesis class $\mc{H}$ is $\Sigma^0_2$-represented if there is a $\Sigma^0_2$ listing of the hypotheses in $\mc{H}$, i.e., a $\Sigma^0_2$ set $U$ such that $\mc{H} = \{ \Phi_e \; \mid \; e \in U\}$.
     \end{definition}

     \begin{question}
         Is there a natural hypothesis class which is $\Sigma^0_2$-represented and PAC learnable but not CPAC learnable?
     \end{question}

    For $G_\delta$ hypothesis classes, we can construct an example where for all $x$ the hypothesis class $\mc{H}_x$ has no $x$-computable hypotheses. This is in some sense cheating as $\mc{H}_x$ has no hope of being CPAC$^x$-learnable. We leave it as an open question (Question \ref{qestion-pi2} below) whether there is a more reasonable counterexample.

    \begin{theorem}
        There is a degree-invariant Borel family of hypothesis classes $\mc{H}_x$ such that for all $x$, $\mc{H}_x$ is $\Pi^0_2$ and PAC-learnable, but $\mc{H}_x$ is not CPAC-learnable on any cone.
    \end{theorem}    
    \begin{proof}
        Let $\mc{F}$ be the set of all functions $f \colon \mathbb{Q} \to \{0,1\}$ such that if $p < q$, then if $f(p) = 1$ then $f(q) = 1$. Note that $\mc{F}$ is closed. It is also easy to see that $\mc{F}$ has VC dimension 1 as given any two elements $p < q$ there is no $f \in \mc{F}$ with $f(p) = 1$ and $f(q) = 0$. Now let $\mc{H}_x = \mc{F} - \{ h : \text{$h$ is $x$-computable}\}$. For all $x$, $\mc{H}_x \subseteq \mc{F}$ is PAC-learnable. Since there are only countably many computable $x$-computable functions, and any countable set is $F_\sigma$, $\mc{H}_x$ is $\bfPi^0_2$. But clearly for any $x$ the hypothesis class $\mc{H}_x$ cannot be properly $CPAC^x$-learnable because it does not contain any $x$-computable functions.
    \end{proof}

    \begin{question}\label{qestion-pi2}
        If a natural hypothesis class is $\Pi^0_2$-represented, and the $x$-computable functions are dense, must it be CPAC learnable?
    \end{question}

    There are also natural questions about whether, e.g., a counterexample to Theorem \ref{thm:det} or Corollary \ref{cor:det} can be found in some model of ZFC, e.g., under $V = L$.

    \bibliography{References}
	\bibliographystyle{alpha}
 
\end{document}